\documentclass[11pt]{article}

\usepackage[a4paper,margin=1in]{geometry}
\usepackage{amsmath,amssymb,amsthm,mathtools}
\usepackage{booktabs,tabularx,array}
\usepackage{graphicx}
\usepackage{caption}
\usepackage{hyperref}
\hypersetup{hidelinks}
\usepackage[numbers,sort&compress]{natbib}
\usepackage{microtype}
\graphicspath{{figures/}}
\newtheorem{assumption}{Assumption}
\newtheorem{proposition}{Proposition}

\newcommand{\diag}{\mathrm{diag}}
\newcommand{\E}{E}

\title{DFSC: Error-Controlled Differentiable Mittag--Leffler Propagation for Fractional Scientific Machine Learning}
\author{
Ning Hu\textsuperscript{1,*}
\and Haitao Duan\textsuperscript{2}
\and Shuqun Li\textsuperscript{3}
\and Chuyang Hu\textsuperscript{4}
\\[0.6em]
\small \textsuperscript{1}School of Mechanical Engineering, Hangzhou Dianzi University\\
\small 1158 No. 2 Street, Qiantang District, Hangzhou 310018, China\\
\small \textsuperscript{2}School of Computer Science, Hangzhou Dianzi University\\
\small 1158 No. 2 Street, Qiantang District, Hangzhou 310018, China\\
\small \textsuperscript{3}School of Mathematics and Physics, Xi'an Jiaotong-Liverpool University\\
\small 111 Ren'ai Road, Suzhou 215123, China\\
\small \textsuperscript{4}Department of Computer Science and Engineering, University of California, Santa Cruz\\
\small 1156 High Street, Santa Cruz, CA 95064, USA\\
\small \textsuperscript{*}Corresponding author: \href{mailto:jluhooning@163.com}{jluhooning@163.com}\\
\small Haitao Duan: \href{mailto:25050509@hdu.edu.cn}{25050509@hdu.edu.cn}\\
\small Shuqun Li: \href{mailto:Shuqun.Li24@student.xjtlu.edu.cn}{Shuqun.Li24@student.xjtlu.edu.cn}\\
\small Chuyang Hu: \href{mailto:chu207@ucsc.edu}{chu207@ucsc.edu}\\
\small ORCID: Ning Hu \href{https://orcid.org/0000-0002-2718-0385}{0000-0002-2718-0385}\\
\small Haitao Duan \href{https://orcid.org/0009-0006-4611-8639}{0009-0006-4611-8639};
Shuqun Li \href{https://orcid.org/0009-0007-9255-018X}{0009-0007-9255-018X}\\
\small Chuyang Hu \href{https://orcid.org/0009-0007-1787-685X}{0009-0007-1787-685X}
}
\date{}

\begin{document}
\maketitle

\begin{abstract}
Fractional scientific machine learning requires numerical operators that can be differentiated, batched, accelerated, and composed with neural networks. When the dominant linear fractional evolution is known through a Mittag--Leffler propagator, repeatedly reconstructing that response with a history solver or relearning it from data is unnecessary. We present DFSC, a PyTorch environment organized around the Mittag--Leffler Spectral Layer (MLSL). The layer separates known fractional propagation from data-driven corrections, so neural modules learn only unresolved dynamics while fractional orders and residual-network parameters are optimized jointly. Its adaptive algorithm increases special-function truncation depth or Lanczos dimension until successive differentiable evaluations satisfy a requested tolerance. In the negative-real alternating-series regime, DFSC additionally returns a certified first-omitted-term bound; outside that regime it explicitly labels estimates as empirical.

DFSC supports dense, sparse, matrix-free, self-adjoint, generalized, and controlled complex operator paths; trainable fractional orders; direct inverse problems; residual neural composition; and CPU/GPU execution. The certified series bound covers all 59 eligible reference cases, with median bound/error effectivity 1.246 for resolved errors. Reusing a prepared batched Lanczos basis gives identical fixed-path values and reduces repeated-query time by $4.61$--$7.11\times$ on CPU and $13.07$--$16.22\times$ on an RTX~5070, excluding one-time preparation. A 27-case inverse matrix finds full-rank local curvature throughout, while remaining explicitly model-conditional. External solver and mixed real-data results support DFSC as an error-aware optional primitive for matched fractional structure, rather than a general replacement for fractional solvers or neural models.
\end{abstract}


\section{Introduction}

Fractional differential equations describe memory, anomalous transport, and nonlocal response in a compact form \citep{podlubny1999fractional,kilbas2006theory,diethelm2010analysis}. Their numerical treatment is mature enough to include convolution quadrature, L1-type schemes, predictor--corrector methods, and spectral discretizations \citep{lubich1986discretized,jin2016l1,lischke2020fractional}. Scientific machine learning (SciML), however, changes the software contract. A numerical operator used during training must propagate derivatives with respect to model parameters, preserve batch and device semantics, expose diagnostics, and compose with neural modules without moving data outside the computational graph \citep{baydin2018autodiff,innes2019differentiable,rackauckas2020universal}.

Three established software directions address related parts of this problem. General equation-solving ecosystems provide broad problem and algorithm coverage. Fractional solver libraries focus on robust time integration for several definitions and equation classes. Neural SciML libraries provide physics-informed residuals and learned operators, including fPINNs, DeepONets, and Fourier neural operators \citep{pang2019fpinn,lu2021deeponet,li2021fno,lu2021deepxde,kovachki2023neuraloperator}. These directions remain the appropriate choices when a general time-stepper, an unknown solution operator, or a flexible residual learner is required.

This work addresses a narrower but recurring setting. After spatial discretization or modal reduction, a linear Caputo system can take the nondimensional form
\begin{equation}
{}^{C}D_t^{\alpha}u(t)+A_{\beta}u(t)=f(t),
\qquad 0<\alpha\leq 2,
\label{eq:governing}
\end{equation}
where time, state, and operator scales have been absorbed into the nondimensional variables. When $A_{\beta}$ is diagonalizable, its homogeneous solution is governed by a Mittag--Leffler spectral propagator. The dominant evolution is then known, while fractional orders, forcing, residual physics, or observation maps may remain unknown. Replacing this structure by a generic network discards useful inductive bias; differentiating through a full history solver can be unnecessary when only selected query times are required.

DFSC \citep{hu2026dfsc} turns this known propagator into a software ecosystem. Its core component, MLSL, behaves as a PyTorch layer: it accepts batched tensors, keeps $\alpha$ and $\beta$ trainable, runs on CPU or GPU, and supports residual neural heads. Around the primitive, DFSC adds operator construction, sparse and matrix-free actions, direct and history-aware algorithms, automatic selection, reliability reports, inverse workflows, application templates, and release-oriented validation. Figure~\ref{fig:ecosystem} summarizes this organization.

From a neurocomputing perspective, DFSC is a structured learning component rather than a standalone time integrator: fractional orders, propagator parameters, and neural residuals participate in joint gradient-based optimization.

\begin{figure}[t]
\centering
\includegraphics[width=0.98\linewidth]{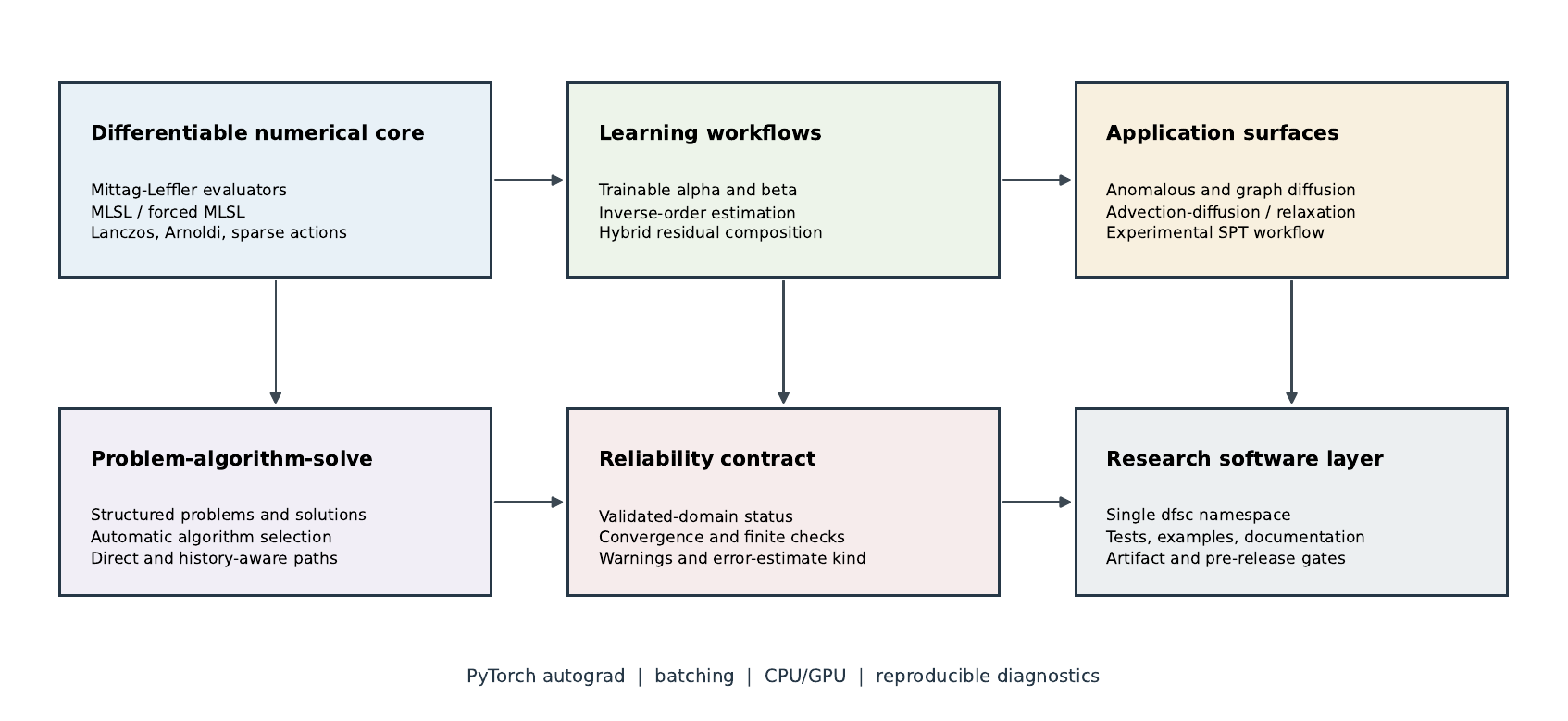}
\caption{DFSC architecture. A differentiable Mittag--Leffler core supports learning workflows and a problem--algorithm--solve software layer. Reliability metadata accompanies solutions instead of being inferred from successful execution alone.}
\label{fig:ecosystem}
\end{figure}

The contribution is not a new fractional derivative or a replacement for general fractional solvers. It is a computational abstraction and its supporting ecosystem. Specifically, this paper contributes:
\begin{enumerate}
\item a trainable PyTorch MLSL primitive and hybrid residual formulation that separate known fractional propagation from unknown corrections while exposing fractional orders and neural parameters to joint optimization across dense, sparse, matrix-free, generalized, and controlled complex paths;
\item an adaptive differentiable propagation algorithm with explicit component-wise error budgets, a certified alternating-series remainder in its stated domain, and empirical diagnostics elsewhere;
\item a unified \texttt{dfsc} package with problem--algorithm--solve semantics, automatic selection, explicit applicability and reliability contracts, history-aware fallbacks, application templates, and reproducibility tooling; and
\item a layered evidence suite spanning numerical identities and gradients, tolerance calibration, external FDEint/pycaputo benchmarks, synthetic learning tasks, 12 real experimental conditions in four physical domains, and software-artifact gates.
\end{enumerate}

\section{Related Software and Methods}

\subsection{Differentiable solvers and neural SciML ecosystems}
DifferentialEquations.jl and SciMLSensitivity demonstrate how common problem interfaces, solver selection, sensitivity algorithms, and solution metadata can support a large SciML ecosystem \citep{rackauckas2017differentialequations,rackauckas2020universal}. DeepXDE combines physics-informed learning, operator learning, multiple backends, geometries, and boundary conditions \citep{lu2021deepxde}. NeuralOperator supplies reusable neural-operator architectures and training infrastructure \citep{kossaifi2024neuraloperator}. DFSC adopts the ecosystem principles of composability, explicit algorithms, and structured solutions, but specializes them to differentiable fractional propagation in PyTorch.

\subsection{Fractional learning and operator approximation}
fPINNs impose fractional residuals and have been extended to Monte Carlo discretizations and inverse settings \citep{pang2019fpinn,guo2022mcpinn,sheng2024mcfpinn}. FNO and DeepONet learn mappings between function spaces \citep{li2021fno,lu2021deeponet}. In contrast, MLSL does not learn a known propagator. It evaluates the retained propagator and exposes unknown orders or residual terms to optimization. Hybrid composition is therefore most useful when the analytical backbone is credible but incomplete.

\subsection{Fractional numerical solvers}
Classical methods remain indispensable when a complete trajectory or a general fractional equation must be advanced in time. L1 methods approximate Caputo history integrals on a grid \citep{jin2016l1}; convolution quadrature provides a general discretization framework \citep{lubich1986discretized}; and spectral approaches exploit spatial operator structure \citep{lischke2020fractional}. FractionalDiffEq.jl provides a Julia solver suite with a unified interface and multiple fractional equation classes \citep{qu2025fractionaldiffeq}. FDEint targets neural fractional differential equations through a PyTorch predictor--corrector implementation \citep{zimmering2024optimising}. PyCaputo supplies experimental Python tools for derivatives, quadrature, and fractional ODE stepping \citep{pycaputo2026}.

DFSC overlaps with these libraries only on the history-aware fallback path. Its distinguishing object is the directly queried, trainable propagator layer for systems whose Mittag--Leffler representation is available. Consequently, DFSC should complement rather than replace broad solver libraries.

\subsection{Mittag--Leffler evaluation}
Mittag--Leffler functions generalize the exponential response of integer-order systems \citep{gorenflo2014mittag}. Reliable numerical evaluation requires attention to series cancellation, asymptotic validity, and parameter-dependent derivatives \citep{garrappa2015mittag}. DFSC currently validates a declared region rather than claiming global special-function coverage. This restriction is carried into each solution through a reliability report.

\section{Mathematical and Computational Core}

\subsection{Spectral propagation}
Assume that the nondimensional operator in Eq.~\eqref{eq:governing} admits
\begin{equation}
A_{\beta}=\Phi\diag(\mu_1(\beta),\ldots,\mu_M(\beta))\Phi^{-1}.
\label{eq:eigendecomp}
\end{equation}
For the homogeneous problem, the retained modal solution is
\begin{equation}
u_M(t)=\Phi\diag\!\left[\E_{\alpha}\!\left(-\mu_n(\beta)t^{\alpha}\right)\right]\Phi^{-1}u_0.
\label{eq:homogeneous}
\end{equation}
For forcing $f$, the mild form is
\begin{equation}
u(t)=\E_{\alpha}(-A_{\beta}t^{\alpha})u_0+
\int_0^t (t-s)^{\alpha-1}\E_{\alpha,\alpha}\!\left[-A_{\beta}(t-s)^{\alpha}\right]f(s)\,ds.
\label{eq:forced}
\end{equation}
All terms in Eqs.~\eqref{eq:governing}--\eqref{eq:forced} are nondimensional. In dimensional applications, $A_{\beta}t^{\alpha}$ is replaced by the corresponding dimensionless combination after selecting characteristic time and operator scales.

Equation~\eqref{eq:homogeneous} defines the MLSL forward map. The implementation evaluates modal multipliers for all batch elements and query times using tensor operations. The graph shown in Fig.~\ref{fig:graph} preserves derivatives with respect to $\alpha$, $\beta$, operator parameters, inputs, and residual-network weights.

\begin{figure}[t]
\centering
\includegraphics[width=0.92\linewidth]{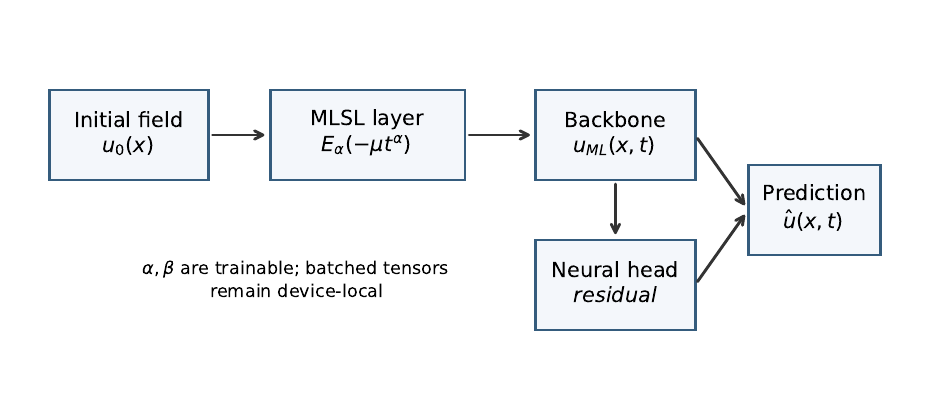}
\caption{MLSL in a hybrid neural computation graph. The known propagator supplies a structured backbone; a neural head represents only the unresolved correction.}
\label{fig:graph}
\end{figure}

\subsection{Evaluator and gradient contract}
DFSC evaluates $\E_{\alpha}$ and $\E_{\alpha,\beta}$ with regime-aware series and asymptotic branches. For radius $r=|z|$ in a transition interval $[r_-,r_+]$, it uses
\begin{equation}
q=\operatorname{clip}\!\left(\frac{r-r_-}{r_+-r_-},0,1\right),\qquad
s(q)=3q^2-2q^3,\qquad
E_{\mathrm{blend}}=(1-s)E_{\mathrm{series}}+sE_{\mathrm{asymp}}.
\label{eq:blend}
\end{equation}
Because $s'(0)=s'(1)=0$, the weight itself joins with zero endpoint slope and reduces artificial derivative jumps caused by a hard switch. This does not prove globally smooth approximation error: branch approximations need not agree exactly, and routing thresholds may depend on detached policy variables. The evaluator reports whether arguments lie inside the tested domain, whether the iteration converged, whether values are finite, and what kind of error estimate is available. Strict mode rejects inputs outside the declared region.

Trainable orders are represented by unconstrained variables passed through bounded transforms. For a loss $L$, PyTorch computes
\begin{equation}
\frac{\partial L}{\partial \theta_{\alpha}}=
\frac{\partial L}{\partial u_M}
\frac{\partial u_M}{\partial \alpha}
\frac{\partial \alpha}{\partial \theta_{\alpha}},
\qquad
\frac{\partial L}{\partial \theta_{\beta}}=
\frac{\partial L}{\partial u_M}
\frac{\partial u_M}{\partial \beta}
\frac{\partial \beta}{\partial \theta_{\beta}}.
\label{eq:chain}
\end{equation}
No detached special-function call is used on the validated real-valued path.

The adaptive controller evaluates a nested work schedule $w_1<\cdots<w_J$. For either truncation depth or Krylov dimension, it selects the first $w_j$ satisfying
\begin{equation}
\eta_j=\frac{\|u_{w_j}-u_{w_{j-1}}\|}{\max(\|u_{w_j}\|,\epsilon)}
\leq \mathrm{rtol}+\frac{\mathrm{atol}}{\max(\|u_{w_j}\|,\epsilon)}.
\label{eq:adaptive}
\end{equation}
The stop decision is detached, but $u_{w_j}$ retains its computational graph. Thus gradients are computed through the selected candidate, not through the discrete decision. The resulting map is piecewise differentiable on regions where the selected budget is constant; a gradient jump may occur on a budget-selection boundary. The indicator $\eta_j$ measures successive numerical disagreement; it is not claimed to bound spectral truncation, model reduction, or systematic special-function error.

For $z=-x\leq0$, $0<\alpha\leq1$, and $\beta>0$, write the retained series as $\sum_{k=0}^{N-1}(-1)^k a_k$ with $a_k=x^k/\Gamma(\alpha k+\beta)$. Once $a_N\leq a_{N-1}$, the gamma-ratio monotonicity places the omitted tail in the decreasing alternating regime, and
\begin{equation}
\left|E_{\alpha,\beta}(-x)-\sum_{k=0}^{N-1}\frac{(-x)^k}{\Gamma(\alpha k+\beta)}\right|
\leq \frac{x^N}{\Gamma(\alpha N+\beta)}.
\label{eq:alternating-bound}
\end{equation}
DFSC reports this quantity as a rigorous series-truncation bound only when the stated conditions are verified. Krylov and spatial-projection terms remain separately assessed; an unavailable component is recorded as missing rather than zero.

\subsection{Operator actions and algorithm paths}
Explicit diagonalization is suitable for regular spectral bases and moderate dense systems. DFSC also provides fixed and adaptive Lanczos actions for self-adjoint operators, Arnoldi actions for a controlled complex regime, and sparse or matrix-free contracts that avoid materializing dense matrices. A Caputo--L1 stepper and an FFT-accelerated full-trajectory L1 operator cover problems for which direct propagation is unavailable or a complete history is required. The FFT operator is an offline trajectory operator, not an online implicit integrator.

\subsection{Theoretical properties}
\begin{assumption}[Retained spectral representation]
The operator is diagonalizable on the retained subspace, and the evaluated arguments lie inside the reliability domain declared by the selected evaluator.
\end{assumption}

\begin{proposition}[Spectral consistency]
Under the retained spectral representation, MLSL applies the exact retained modal Mittag--Leffler evolution up to spatial truncation, special-function approximation, and floating-point error.
\end{proposition}
\begin{proof}
Project $u_0$ onto the retained basis, solve each scalar Caputo mode by its Mittag--Leffler response, and reconstruct through $\Phi$. The remaining errors arise only from the stated numerical approximations. A detailed decomposition is given in Appendix~\ref{app:proofs}.
\end{proof}

\begin{proposition}[History-free query complexity]
For $B$ initial states, $Q$ query times, and $M$ retained diagonal modes, modal multiplication costs $O(BQM)$ after basis transforms and requires no recurrent temporal state. History methods retain their separate advantages for sequential forcing, nonlinear stepping, and models without a usable propagator.
\end{proposition}

\begin{proposition}[Differentiability on a fixed numerical branch]
If the selected evaluator branch is differentiable and finite in an open neighborhood of the parameters, the MLSL map is differentiable with respect to trainable orders and differentiable operator parameters in that neighborhood.
\end{proposition}

\begin{proposition}[Piecewise differentiability under adaptive work control]
Let every candidate map $u_{w_j}(\vartheta)$ be continuously differentiable on an open parameter set. On any open subset where Eq.~\eqref{eq:adaptive} selects the same $w_j$, the adaptive output is continuously differentiable and its derivative equals $\partial u_{w_j}/\partial\vartheta$. This statement makes no continuity claim across a boundary where the selected budget changes.
\end{proposition}

\section{The DFSC Software Ecosystem}

\subsection{Three-layer organization}
The package turns the mathematical core into three user-facing layers (Table~\ref{tab:layers}). This organization corresponds to increasingly complete research workflows without exposing internal project planning in the public API.

\begin{table}[t]
\centering
\caption{DFSC software layers and their evidence.}
\label{tab:layers}
\small
\begin{tabularx}{\linewidth}{@{}p{0.17\linewidth}p{0.27\linewidth}X p{0.22\linewidth}@{}}
\toprule
Layer & Main objects & Role & Validation \\
\midrule
Numerical primitive & MLSL, forced MLSL, scalar evaluators, Krylov actions & Differentiable known propagation & Identities, references, gradients, branch probes \\
Learning workflows & Trainable orders, inverse solver, residual regressor, application templates & Estimate orders and combine known and learned terms & Multi-seed inverse, hybrid and real-data studies \\
Software ecosystem & Problem, algorithm, solve, selector, solution, reliability, registry & Reusable and diagnosable execution & Unit tests, API smoke tests, artifact and release gates \\
\bottomrule
\end{tabularx}
\end{table}

\subsection{Problem--algorithm--solve interface}
A \texttt{Problem} stores the operator, initial state, times, orders, forcing, and structural assumptions. An \texttt{Algorithm} selects direct spectral, stable spectral, Lanczos, Arnoldi, L1, FFT history, or semilinear execution. \texttt{solve(problem, algorithm)} returns values, status, statistics, warnings, and a \texttt{ReliabilityReport}. Automatic selection uses declared operator structure and workload requirements; it does not infer mathematical validity from data alone.

\subsection{Applicability and reliability}
The applicability contract distinguishes supported, experimental, and unsupported settings. Variable- and distributed-order wrappers are marked experimental. Complex arguments are accepted only through the controlled complex evaluator and Arnoldi path. A successful return code indicates successful execution, whereas the reliability level records validated-domain membership, finite values, convergence, gradient status, and the availability of an empirical or rigorous error estimate. The current implementation explicitly reports that no global rigorous solver error bound is available.

\subsection{Learning composition}
DFSC supports direct parameter recovery and hybrid models of the form
\begin{equation}
\widehat{u}(x,t)=u_{\mathrm{MLSL}}(x,t;\alpha,\beta,\vartheta)+r_{\theta}(x,t,u_{\mathrm{MLSL}}),
\label{eq:hybrid}
\end{equation}
where $\vartheta$ denotes operator parameters and $r_{\theta}$ is a neural correction. The decomposition is useful when retained linear dynamics are trusted but forcing, constitutive response, or unresolved scales are not.

\subsection{Scope of version 0.1.0}
The single distributed namespace is \texttt{dfsc}. Version 0.1.0 contains 22 implemented components and two experimental wrappers. It is a specialized beta: it is not yet distributed on PyPI, has no hosted API site or independent adopters, and is not a general fractional solver. Its strongest current domain is differentiable propagation for known or approximated Mittag--Leffler operator structure.

\section{Experimental Design}

The evidence suite is organized around numerical fidelity, error certification, adaptive work control, differentiability, inverse identifiability, repeated-query efficiency, external-software comparison, learning utility, and real-task validity. Unless stated otherwise, computations use PyTorch 2.11.0 with CUDA 12.8; GPU tests use an NVIDIA GeForce RTX 5070 Laptop GPU. Synthetic studies use manufactured solutions so that reference values and parameters are known.

\subsection{Numerical core}
Tests include $E_1(z)=e^z$, $E_2(-x^2)=\cos x$, $E_{1,1}(z)=e^z$, high-precision scalar references, dense matrix references, and autograd--finite-difference comparisons. Transition probes cover 63 cases around numerical branch boundaries. FFT history, Lanczos, Arnoldi, and matrix-free actions are compared with direct or dense references.

\subsection{Learning workflows}
Synthetic inverse tests estimate $\alpha$ and $\beta$ from complete, noisy, and sparse observations. Hybrid tests compare a retained MLSL backbone, a residual MLSL model, FNO, and DeepONet on a controlled known-propagator dataset. The neural controls use 126 training pairs and three seeds. A budget sensitivity study retains each architecture, optimizer, learning rate, data split, and initialization while increasing FNO updates from 90 to 360 and DeepONet updates from 120 to 480. A separate sample-efficiency matrix uses nested training sets of 16, 32, 64, and 128 fields and evaluates IID, time-, order-, forcing-, and joint-OOD regimes over three seeds. These baselines diagnose learning burden in the matched setting rather than universal architectural superiority: the data generator contains the exact retained backbone.

\subsection{External solvers and real experimental tasks}
The external software benchmark solves both ${}^{C}D_t^{\alpha}y=-y$, $y(0)=1$ and a stable coupled two-state linear system for $\alpha\in\{0.55,0.75,0.95\}$, using up to 1,025 query points. It compares DFSC adaptive direct queries with FDEint~0.1.1 predictor--corrector stepping and pycaputo~0.10.2 PECE under recorded package versions and identical CPU precision. Because only DFSC exploits the known propagator, timing results quantify a specialized-use advantage rather than general solver superiority. Direct queries are useful at selected times and under repeated parameter updates; continuous nonlinear forcing or path-dependent stepping generally requires the L1/FFT or semilinear workflows and forfeits this direct-query advantage.

The first real domain uses AnomDiffDB trajectories \citep{granik2019single,anomdiffdb}. The 750-nm H-actin condition contains 3,112 trajectories; a separate water--glycerol dataset contributes 660 trajectories longer than 40 frames, partitioned deterministically into slow (104) and fast (556) empirical populations from log median one-frame diffusivity. Five 70/30 trajectory splits fit lags 1--20 and extrapolate to lags 21--40. The second domain uses four geomembrane stress-relaxation tests at stretch ratios 1.001--1.008 from an EPFL CC BY 4.0 dataset \citep{vorlet2025geomembrane}. Ten-second median bins are sampled at 120 times; the first 60\% are used for fitting and the final 40\% for extrapolation. Three seeds compare exponential, stretched-exponential, fractional-Zener/MLSL, pure-MLP, and MLSL--residual models.

The third domain uses four thermocouple histories from two consecutive charging--discharging cycles of a pilot-scale geopolymer thermal-energy-storage prototype \citep{georgiou2026geotes}. The CC BY 4.0 workbook is aligned at 445 common timestamps without interpolation or smoothing. T1 is treated as a measured thermal driver and T2--T4 as ordered response channels. The first cycle is used for identification and the second for transfer, with 56 samples retained per cycle. The source does not document sensor coordinates in the workbook or README; accordingly, this experiment evaluates forced multi-channel transfer and does not claim spatial-field reconstruction.

The fourth domain uses 32,370 measured temperatures from 16 heated-steam injection experiments at ten documented depths \citep{garciamartinez2025heatedsteam}. The source workbook is converted to long form without interpolation or smoothing. Experiments 5, 8, 12, and 16 are held out because they contain the largest flow, inlet temperature, column height, and water content, respectively. After a fixed five-minute sampling rule, 5,090 training and 1,470 held-out rows compare a matched integer forced spectral model, a trainable fractional model, a pure MLP, and a smaller fractional--residual model over three seeds.

\section{Results}

\subsection{Numerical reliability and gradients}
Table~\ref{tab:numerics} reports representative numerical checks. The restricted certificate in Eq.~\eqref{eq:alternating-bound} is tested separately on 60 combinations of $\alpha$, $|z|$, and truncation depth. Fifty-nine satisfy its eligibility test; all 59 bounds contain the richer-series reference error. For errors above $10^{-14}$, the median ratio of bound to observed error is 1.246. The remaining case is reported as uncertified because the terms have not yet entered the decreasing regime. The alpha derivative agrees with finite differences to $2.17\times10^{-10}$ relative discrepancy in the pre-release gate. FFT history and matrix-free paths agree with their references near floating-point precision. Controlled complex evaluation is validated only for $|z|\leq4$ and Arnoldi reduced radius at most 4.

\begin{table}[t]
\centering
\caption{Representative numerical validation results. Values are relative errors unless noted.}
\label{tab:numerics}
\small
\begin{tabularx}{\linewidth}{@{}Xrr@{}}
\toprule
Check & Result & Declared scope \\
\midrule
Alpha autograd versus finite difference & $2.17\times10^{-10}$ & Pre-release scalar case \\
FFT versus direct Caputo--L1 trajectory & $5.09\times10^{-16}$ & Uniform-grid offline trajectory \\
Sparse/matrix-free versus dense action & $2.09\times10^{-15}$ & 64-state reference \\
Lanczos full-dimension action & $1.83\times10^{-15}$ & Self-adjoint reference \\
Complex scalar Mittag--Leffler & $1.39\times10^{-15}$ & $|z|\leq4$ \\
Arnoldi matrix action & $6.86\times10^{-17}$ & Reduced radius $\leq4$ \\
\bottomrule
\end{tabularx}
\end{table}

The 63 transition probes return finite values and gradients in every case. Their maximum adjacent value and gradient changes are $5.20\times10^{-3}$ and $4.43\times10^{-2}$, respectively. These measurements support usable local behavior but do not establish a global smoothness theorem.

Figure~\ref{fig:adaptive-gradient} isolates the separate effect of adaptive work selection. Across 185 values of $\alpha\in[0.52,0.98]$, the controller changes its series budget three times; all evaluations converge and all gradients are finite. The adaptive objective and its $\alpha$ derivative agree with a fixed 180-term series to double-precision resolution at every probe, including the three observed switches. In six inverse runs, gradients and losses remain finite for Adam learning rates $10^{-3}$, $5\times10^{-3}$, and $2\times10^{-2}$; under a fixed 160-update budget, the largest learning rate reaches the target from both initializations, whereas smaller rates progress more slowly. The experiment therefore supports stable selected-path differentiation in the tested regime, not learning-rate invariance or global smoothness across every possible selection boundary.

\begin{figure}[t]
\centering
\includegraphics[width=0.96\linewidth]{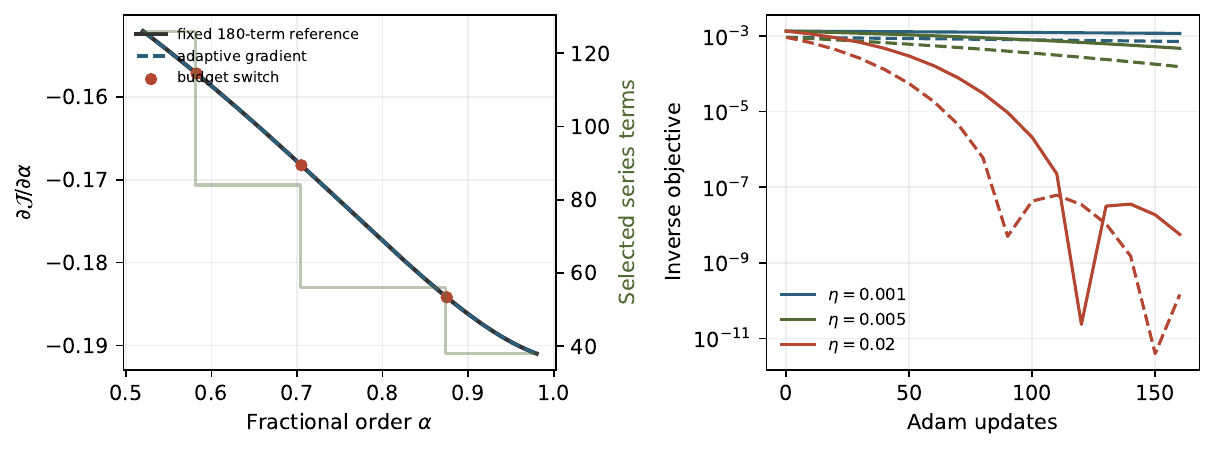}
\caption{Adaptive-gradient diagnostics. Left: the selected-path gradient and a fixed 180-term reference across three observed work-budget switches. Right: inverse objectives for two initial orders at three learning rates under the same 160-update budget.}
\label{fig:adaptive-gradient}
\end{figure}

\subsection{Adaptive work control and external software}
Figure~\ref{fig:adaptive} reports 30 GPU calibrations on weighted Laplacians of size 64 and 128 over five seeds. Mean selected Krylov dimension rises from 13.2 at tolerance $10^{-3}$ to 20.8 at $10^{-5}$ and 41.6 at $10^{-7}$. Convergence rates are 100\%, 100\%, and 80\%, respectively; nonconverged cases return an explicit exhausted-schedule status. Mean actual errors are $6.20\times10^{-6}$, $8.40\times10^{-7}$, and $2.74\times10^{-8}$, and all 30 errors are below $10\,\mathrm{rtol}$. This is empirical calibration of Eq.~\eqref{eq:adaptive}, not a proof of a posteriori reliability.

\begin{figure}[t]
\centering
\includegraphics[width=0.96\linewidth]{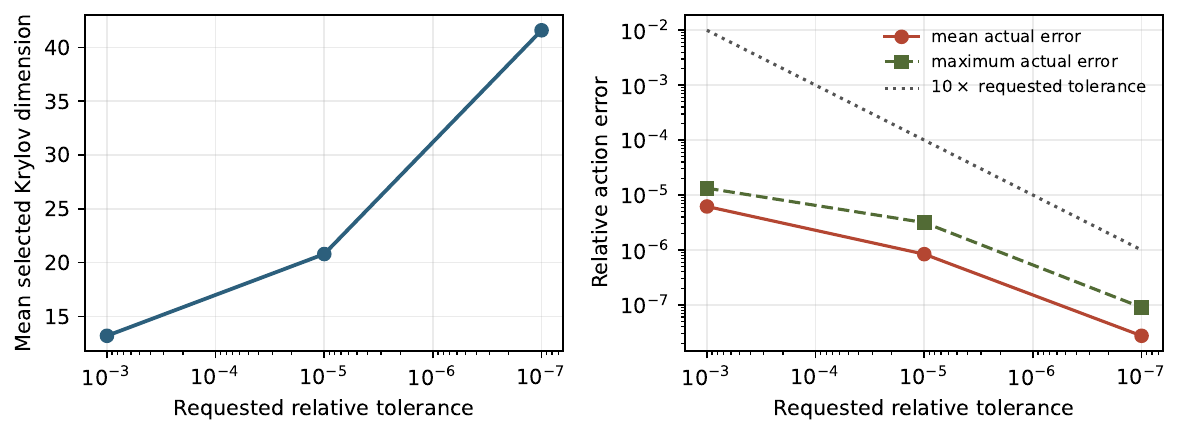}
\caption{Adaptive Lanczos calibration on an RTX 5070 Laptop GPU. Tighter requested tolerances select larger subspaces and reduce dense-reference action error. Successive-action disagreement remains an empirical indicator rather than a rigorous bound.}
\label{fig:adaptive}
\end{figure}

At 1,025 query points, DFSC takes 1.7--3.0 ms across the three orders, pycaputo PECE takes 71.7--76.3 ms, and FDEint takes 154--159 ms. DFSC relative errors against pymittagleffler are below $10^{-15}$ in this moderate negative-real regime; pycaputo errors range from $5.0\times10^{-7}$ to $1.36\times10^{-5}$ and FDEint errors from $9.67\times10^{-4}$ to $1.18\times10^{-3}$. Both DFSC and FDEint return finite gradients with respect to $\alpha$; pycaputo uses NumPy stepping. The benchmark favors DFSC by construction because the analytical propagator is known, which is precisely the intended decision boundary.

The coupled system gives the same ordering without reducing the task to independent scalar input/output. At 1,025 points, DFSC takes 2.9--4.8 ms with relative error between $6.6\times10^{-15}$ and $1.4\times10^{-12}$; pycaputo takes 124--128 ms with errors between $3.0\times10^{-6}$ and $7.4\times10^{-5}$; and FDEint takes 178--188 ms with errors between $2.92\times10^{-3}$ and $3.22\times10^{-3}$. This extension tests state coupling and modal reconstruction while retaining the same scoped known-propagator premise.

\begin{table}[t]
\centering
\caption{External-package benchmark at 1,025 query points over $\alpha\in\{0.55,0.75,0.95\}$. Entries are ranges of three-repeat median runtimes and relative errors.}
\label{tab:external}
\small
\begin{tabularx}{\linewidth}{@{}llrr@{}}
\toprule
Task & Method & Runtime (ms) & Relative error \\
\midrule
Scalar & DFSC direct & 1.7--3.0 & $<10^{-15}$ \\
Scalar & pycaputo PECE & 71.7--76.3 & $5.0\times10^{-7}$--$1.36\times10^{-5}$ \\
Scalar & FDEint PC & 154--159 & $9.67\times10^{-4}$--$1.18\times10^{-3}$ \\
Coupled & DFSC direct & 2.9--4.8 & $6.6\times10^{-15}$--$1.4\times10^{-12}$ \\
Coupled & pycaputo PECE & 124--128 & $3.0\times10^{-6}$--$7.4\times10^{-5}$ \\
Coupled & FDEint PC & 178--188 & $2.92\times10^{-3}$--$3.22\times10^{-3}$ \\
\bottomrule
\end{tabularx}
\end{table}

\subsection{Scalability and accelerator execution}
Prepared Lanczos spaces make the existing within-call time batching reusable across separate order and time queries. Across state sizes 64, 128, and 256, batches of four initial states, and 4--64 order queries, all 18 prepared-path outputs agree exactly with the original fixed-path implementation. Excluding one-time basis construction, median-timed query speedups range from $4.61$ to $7.11\times$ on CPU and from $13.07$ to $16.22\times$ on the RTX~5070. These values quantify amortization for a fixed operator and initial-state batch; they are not end-to-end solver speedups, and the basis must be rebuilt when either changes.
The FFT Caputo--L1 trajectory operator is 77.9 times faster than the direct convolution implementation at 2,048 steps in the measured environment and remains finite at 65,536 steps. For a 4,096-state matrix-free action, the estimated dense storage is 128 MiB versus 1 MiB for a 32-vector Krylov basis, a 128-fold storage ratio. CPU/GPU relative discrepancy is $3.36\times10^{-7}$ in the broad GPU validation suite; specialized FFT and controlled Arnoldi checks are closer to floating-point precision. All tested CPU/GPU gradients are finite. These measurements characterize this implementation and hardware, not universal speed rankings against external libraries.

Figure~\ref{fig:speed} shows the direct-query advantage in the scalar relaxation benchmark. The comparison is with the bundled Python L1 reference and therefore mixes algorithmic and implementation effects; it should not be read as a comparison with optimized external solvers.

\begin{figure}[t]
\centering
\includegraphics[width=0.76\linewidth]{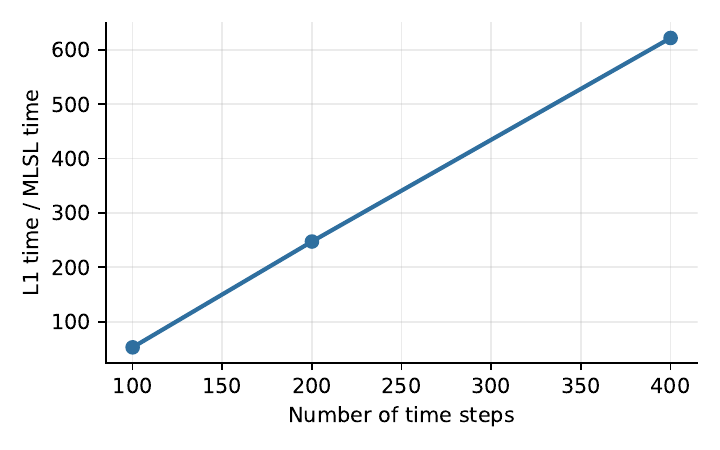}
\caption{Measured direct-query MLSL and bundled L1-reference runtimes. The benchmark illustrates the value of avoiding temporal recurrence when the propagator is known; it is not an external solver leaderboard.}
\label{fig:speed}
\end{figure}

\subsection{Inverse recovery and hybrid learning}
An additional multi-start study evaluates local identifiability for true $(\alpha,\beta)=(0.78,1.15)$ at 12, 32, and 64 sensors and noise standard deviations 0, 0.005, and 0.02. All 27 fitted cases have a positive full-rank Hessian; the maximum Hessian condition number is 8.67, and mean absolute errors are $6.89\times10^{-3}$ for $\alpha$ and $3.64\times10^{-3}$ for $\beta$. The largest absolute local parameter correlation is 0.721. These Hessian and multi-start diagnostics support local recoverability in the manufactured regime but do not prove global identifiability or robustness to model misspecification.
Manufactured joint recovery estimates $\alpha=1.380042$ and $\beta=1.350067$ from true values 1.38 and 1.35. At noise standard deviation $10^{-2}$, relative errors are $1.81\times10^{-3}$ and $1.83\times10^{-3}$. With four sensors and four times, the corresponding errors are $4.18\times10^{-4}$ and $4.44\times10^{-4}$. Figure~\ref{fig:inverse} reports multi-seed recovery.

\begin{figure}[t]
\centering
\includegraphics[width=0.78\linewidth]{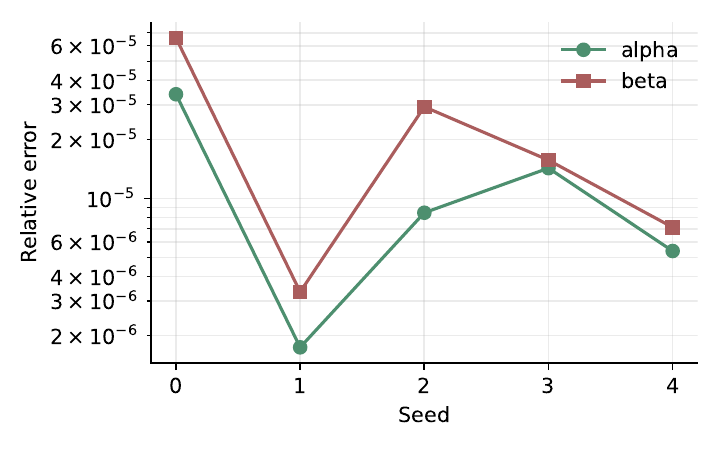}
\caption{Multi-seed inverse recovery of trainable fractional orders under the tested manufactured configurations.}
\label{fig:inverse}
\end{figure}

On the controlled hybrid dataset, the MLSL--residual model attains long-time relative error $0.00326\pm0.00008$, compared with $0.2817\pm0.0012$ for the misspecified MLSL backbone alone. Table~\ref{tab:budget} reports the pure-neural controls under increased update budgets. DeepONet improves modestly at four times the original updates, whereas FNO reduces training error but does not improve long-time extrapolation monotonically. The experiment demonstrates that a known but incomplete propagator can reduce the learning burden in this matched problem. Because the generator contains the same retained fractional structure and no exhaustive architecture search was conducted, these values do not establish general superiority over neural operators.

The nested sample-efficiency matrix gives the same scoped conclusion across five test regimes. With 16 training fields, the pure MLP has IID and joint-OOD errors $0.296\pm0.033$ and $0.545\pm0.124$, whereas MLSL--residual learning gives $0.00462\pm0.00111$ and $0.0480\pm0.0196$. At 128 fields, the corresponding pure-MLP errors are $0.0765\pm0.0078$ and $0.2039\pm0.0373$, compared with $0.000575\pm0.000017$ and $0.0346\pm0.0073$ for the hybrid. Figure~\ref{fig:sample-real} shows the IID and joint-OOD slices. Because the backbone is present in the data generator, this experiment measures sample efficiency under structural match, not general OOD dominance.

\paragraph{Learning-system implications.}
Taken together, the inverse, sample-efficiency, and hybrid experiments identify the learning role of DFSC. When the retained propagator is structurally informative, trainable orders provide low-dimensional physical parameters and the residual network is relieved from relearning the dominant response; this can improve sample efficiency and reduce model size. The real-data results also show the complementary limit: under backbone mismatch, the structured component need not improve predictive accuracy. DFSC therefore supplies a testable inductive bias for gradient-based learning, not a universal advantage over unconstrained neural models.

\begin{figure}[t]
\centering
\includegraphics[width=0.98\linewidth]{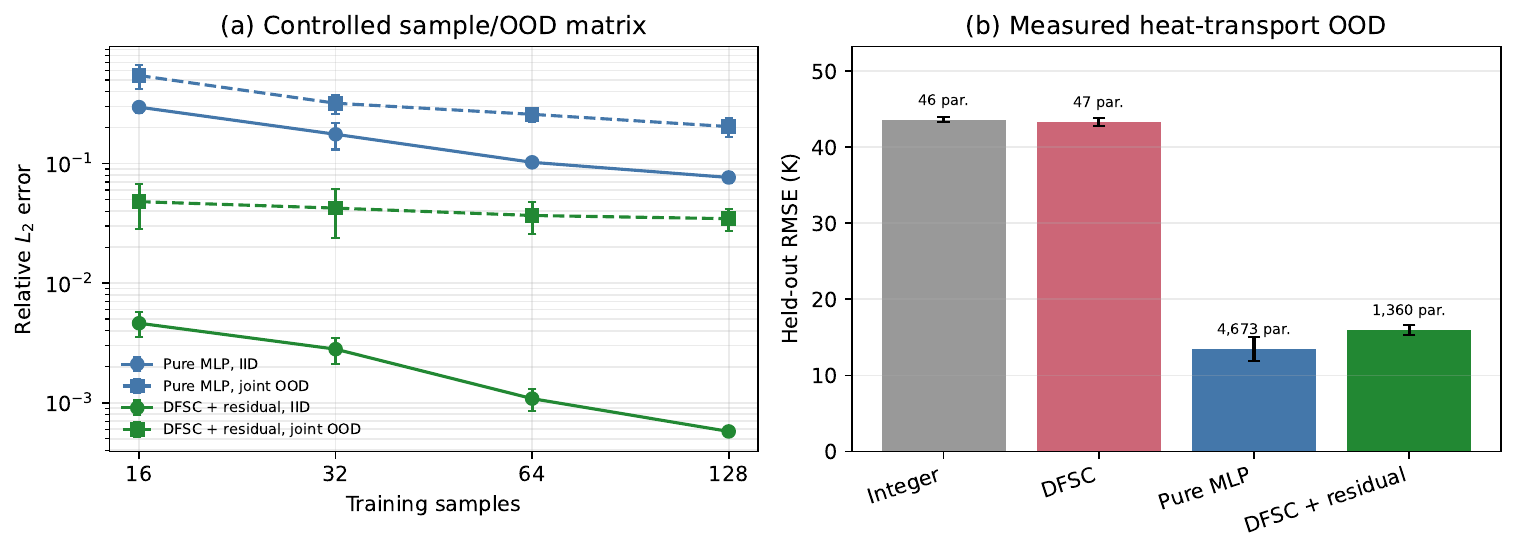}
\caption{Complementary evidence for the structured primitive. Left: sample efficiency and joint OOD behavior on a controlled task containing the retained propagator. Right: measured heated-steam condition OOD, where the pure MLP is most accurate and the smaller hybrid provides a parameter-efficiency tradeoff. Error bars show one standard deviation over three seeds.}
\label{fig:sample-real}
\end{figure}

\begin{table}[t]
\centering
\caption{Training-budget sensitivity of pure-neural diagnostic controls; three seeds, 126 training pairs. Errors are long-time relative $L_2$ mean $\pm$ standard deviation.}
\label{tab:budget}
\small
\begin{tabularx}{\linewidth}{@{}Xrrrr@{}}
\toprule
Model & Parameters & 1$\times$ updates & 2$\times$ updates & 4$\times$ updates \\
\midrule
FNO1D & 79,169 & $2.149\pm0.324$ & $2.056\pm0.330$ & $2.760\pm0.359$ \\
DeepONet1D & 35,361 & $1.593\pm0.053$ & $1.609\pm0.245$ & $1.502\pm0.230$ \\
\bottomrule
\end{tabularx}
\end{table}

\subsection{Real-data evidence}
Table~\ref{tab:spt} gives the scattering-function results on the H-actin condition. Direct MLSL has the lowest mean late-lag error. The hybrid model is substantially better than a pure MLP but does not improve on direct MLSL, indicating that additional neural flexibility is not automatically beneficial when the retained model already explains the observed response.

\begin{table}[t]
\centering
\caption{Experimental H-actin single-particle-tracking results over five trajectory splits. Errors are relative $L_2$ values; mean $\pm$ standard deviation.}
\label{tab:spt}
\small
\begin{tabularx}{\linewidth}{@{}Xrrr@{}}
\toprule
Scattering model & Parameters & Held-out all lags & Late lags 21--40 \\
\midrule
Stretched exponential & 2 & $0.0239\pm0.0074$ & $0.0674\pm0.0279$ \\
Direct MLSL inverse & 2 & $0.0141\pm0.0080$ & $\mathbf{0.0366\pm0.0189}$ \\
Pure MLP & 1,153 & $0.0535\pm0.0170$ & $0.1652\pm0.0555$ \\
MLSL + residual MLP & 323 & $0.0162\pm0.0063$ & $0.0429\pm0.0151$ \\
\bottomrule
\end{tabularx}
\end{table}

Relative to the stretched exponential, direct MLSL reduces mean late-lag error by 45.7\%. Relative to the pure MLP, the hybrid reduces it by 74.0\%. These comparisons support a scoped application claim: retaining a fractional propagation bias can improve data efficiency and extrapolation for this condition. They do not establish that the fitted order uniquely identifies the underlying stochastic mechanism.

Figure~\ref{fig:realmulti} extends the comparison beyond H-actin. In the fast-Brownian population, direct MLSL recovers $\alpha=0.981\pm0.003$ and reduces mean late-lag error from $0.437\pm0.266$ for the pure MLP to $0.152\pm0.097$. In the slow population, however, the stretched exponential gives the lowest mean late-lag error ($0.0205\pm0.0073$), compared with $0.0281\pm0.0222$ for direct MLSL. The empirical clustering does not provide ground-truth bead labels, so these orders remain model-conditional.

Across the four geomembrane tests and three seeds, mean extrapolation errors are $0.261\pm0.038$ for the exponential, $0.350\pm0.021$ for the stretched exponential, $0.409\pm0.030$ for bare fractional-Zener/MLSL, $0.311\pm0.094$ for the pure MLP, and $0.248\pm0.086$ for MLSL--residual composition. The hybrid has the lowest mean but not a decisive low-variance advantage; its median is only slightly lower than that of the exponential (0.257 versus 0.259). These results support residual composition, while directly rejecting a claim that the bare fractional model is uniformly superior.

Table~\ref{tab:geotes} reports cross-cycle transfer on the GeoTES measurements. Trainable fractional orders reduce the second-cycle error relative to the matched integer-order propagator. The 750-parameter hybrid reaches essentially the same mean error as the 1,251-parameter pure MLP, with overlapping seed variation. This supports parameter-efficient residual composition, not a statistically decisive accuracy advantage. The large error of the bare three-channel propagation model also records unresolved forcing, sensor placement, and heat-loss effects rather than being hidden by the neural correction.

\begin{table}[t]
\centering
\caption{Pilot-scale GeoTES transfer from the first measured charging--discharging cycle to the second. Errors are relative $L_2$ mean $\pm$ standard deviation over three seeds.}
\label{tab:geotes}
\small
\begin{tabularx}{\linewidth}{@{}Xrrr@{}}
\toprule
Model & Parameters & Cycle-2 error & Learned $(\alpha,\beta)$ \\
\midrule
Integer propagation & 1 & $0.892$ & $(1,2)$ \\
DFSC & 3 & $0.667$ & $(0.377,1.563)$ \\
Pure MLP & 1,251 & $0.363\pm0.007$ & -- \\
DFSC + residual MLP & 750 & $0.362\pm0.009$ & $(0.399,1.679)$ \\
\bottomrule
\end{tabularx}
\end{table}

The heated-steam benchmark supplies the first real task in this study with explicit spatial coordinates and condition OOD rather than temporal transfer alone. Table~\ref{tab:heated-steam} shows that bare DFSC reduces mean RMSE by only 0.76\% relative to the matched integer model. The pure MLP is most accurate. The hybrid uses 1,360 parameters, 29\% of the pure MLP's count, but its RMSE is 18.4\% higher. This is evidence for a compact-model tradeoff and for the limits of the selected constant-forcing spectral backbone, not a real-data accuracy advantage.

\begin{table}[t]
\centering
\caption{Condition-OOD prediction of measured heated-steam temperature profiles. Four extreme experimental conditions are held out. Values are mean $\pm$ standard deviation over three seeds.}
\label{tab:heated-steam}
\small
\begin{tabularx}{\linewidth}{@{}Xrrr@{}}
\toprule
Model & Parameters & Held-out RMSE (K) & Held-out MAE (K) \\
\midrule
Integer forced spectral & 46 & $43.64\pm0.38$ & $31.13\pm0.27$ \\
DFSC forced spectral & 47 & $43.31\pm0.52$ & $30.94\pm0.34$ \\
Pure MLP & 4,673 & $\mathbf{13.49\pm1.56}$ & $\mathbf{9.66\pm1.15}$ \\
DFSC + residual MLP & 1,360 & $15.98\pm0.63$ & $11.88\pm0.35$ \\
\bottomrule
\end{tabularx}
\end{table}

\begin{figure}[t]
\centering
\includegraphics[width=0.96\linewidth]{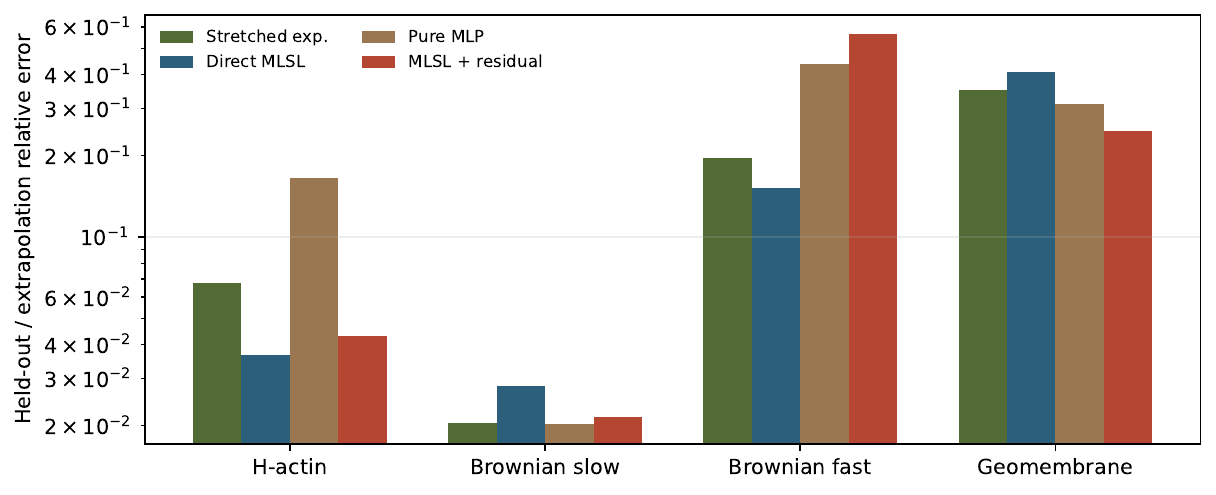}
\caption{Held-out or extrapolation error across four reported real-task groups. No model dominates every condition; the figure motivates DFSC as an optional structured component and shows where a residual head helps or hurts.}
\label{fig:realmulti}
\end{figure}

\subsection{Application and software coverage}
Four application templates test anomalous diffusion, assembled relaxation with a mass matrix, graph diffusion, and advection--diffusion. All produce finite outputs and parameter gradients where applicable; the graph constant mode error is $5.55\times10^{-16}$ and the $\alpha=1$ advection--diffusion exponential error is $1.88\times10^{-16}$. The semilinear Picard workflow converges in three iterations to residual $6.87\times10^{-11}$ in its manufactured test.

The unit-test suite passes all 94 checks, including heated-steam provenance and standardized-data protocol tests, adaptive-control tests, and GeoTES protocol tests (Table~\ref{tab:maturity}). Passing the internal gate means that the checked repository state is internally coherent. It does not imply public-release adoption or production readiness.

\begin{table}[t]
\centering
\caption{Current DFSC maturity assessment.}
\label{tab:maturity}
\small
\begin{tabularx}{\linewidth}{@{}p{0.30\linewidth}p{0.22\linewidth}X@{}}
\toprule
Axis & Status & Evidence or remaining gap \\
\midrule
Differentiable spectral core & Research-grade tested & Identities, references, gradients, CPU/GPU tests \\
Problem/algorithm interface & Public beta & Automatic selection and structured solutions implemented \\
Operator coverage & Scoped broadening & Self-adjoint, generalized, sparse, matrix-free, controlled complex \\
History/nonlinear coverage & Partial & L1, FFT trajectory, and mild-form Picard; not general stepping \\
Applications & Multi-domain experimental & Four templates; SPT, geomembrane, GeoTES, and heated-steam protocols \\
Variable/distributed order & Experimental & Wrappers exist; systematic theory and validation remain \\
Public ecosystem & v0.1.0 released & GitHub, PyPI, hosted documentation, and versioned Zenodo archive; independent adopters remain pending \\
\bottomrule
\end{tabularx}
\end{table}

\section{Discussion}

DFSC is strongest when three conditions hold: a useful part of the dynamics has a Mittag--Leffler representation; selected query times or repeated parameter updates are important; and unknown effects can be represented by trainable orders, forcing, or a residual network. In this regime, the package exposes known propagation directly inside the learning graph and avoids requiring a network to rediscover it.

The ecosystem design is as important as the primitive. A bare layer does not tell a user which algorithm is appropriate, whether an evaluator left its tested domain, or how to reproduce a result. The problem--algorithm--solve interface, applicability report, reliability metadata, application templates, and internal release gate make these assumptions visible. This is the principal advance from an isolated MLSL implementation to DFSC as a coherent research tool.

Comparison with mature ecosystems clarifies the remaining distance. DifferentialEquations.jl has far broader equation, algorithm, sensitivity, and community coverage. FractionalDiffEq.jl addresses more fractional equation classes and time-stepping methods. DeepXDE and NeuralOperator provide substantially richer learning architectures, geometry support, documentation, and user communities. DFSC's advantage is correspondingly specialized: a native PyTorch path from Mittag--Leffler spectral structure to trainable orders, matrix-function actions, reliability metadata, and hybrid neural composition.

\section{Limitations and Threats to Validity}

The current validated domain is narrower than the mathematical domain of Mittag--Leffler functions. Most paths target real non-positive spectral arguments; controlled complex support is limited to the stated radius. Equation~\eqref{eq:alternating-bound} certifies only eligible series truncations; global hybrid-evaluator, Krylov, projection, and solver error bounds have not been proved. Variable- and distributed-order operators remain experimental. Complex geometry, unstructured meshes, and JAX or Julia backends are not implemented.

Synthetic inverse and neural experiments use manufactured data and known operator structure. The FNO and DeepONet budget sweep controls update count but is not an exhaustive architecture or hyperparameter search; the fPINN configuration remains diagnostic. The external FDEint/pycaputo comparison deliberately favors direct propagation and cannot establish general solver speed, particularly for continuous nonlinear forcing or path-dependent integration. The adaptive discrepancy is not a rigorous error bound, the selected map is only piecewise differentiable, and the tightest Krylov schedule converges in 80\% of the calibration runs. Real tasks inherit uncertainty from trajectory clustering, preprocessing, finite samples, load-cell noise, model mismatch, and a limited number of physical systems. The GeoTES channel order is used without undocumented sensor coordinates, so it is a multi-channel transfer test rather than spatial-field validation. The heated-steam split changes both physical conditions and data volume across held-out experiments; its poor bare-spectral accuracy shows that the selected constant-forcing backbone is incomplete. AnomDiffDB files are not redistributed pending license clarification; the geomembrane, GeoTES, and heated-steam data are CC BY 4.0.

Finally, the 100\% internal gate is a repository-specific checklist, not a maturity percentage for the global fractional-solver problem. Version 0.1.0 provides a public package, hosted documentation, and continuous release testing on multiple platforms. Broader production maturity still requires independent reproduction, downstream users, issue handling, and wider benchmark coverage.

\section{Reproducibility, Data, and Code Availability}

The repository contains the \texttt{dfsc} package, unit tests, examples, experiment scripts, environment metadata, result summaries, figure generators, and manuscript assets. Version 0.1.0 is available from GitHub (\url{https://github.com/hzhooning-art/DFSC}), PyPI (\url{https://pypi.org/project/dfsc/0.1.0/}), and the documentation site (\url{https://hzhooning-art.github.io/DFSC/}). The immutable release at commit \texttt{f671e2f63c5ef133491a0dfbb2272baccaac245c} is archived under the version DOI \href{https://doi.org/10.5281/zenodo.21588834}{10.5281/zenodo.21588834}. The reported environment uses Python 3.12.13, PyTorch 2.11.0+cu128, FDEint 0.1.1, and pycaputo 0.10.2. AnomDiffDB, EPFL, GeoTES, and heated-steam checksums, preprocessing rules, split definitions, and raw per-run results are recorded. The heated-steam source workbook and standardized 32,370-row table are retained under the CC BY 4.0 license; AnomDiffDB raw files are excluded because their redistribution terms are not explicit.

\section{Conclusion}

DFSC provides a reusable PyTorch learning environment in which Mittag--Leffler propagation participates directly in automatic differentiation, inverse learning, and hybrid neural modeling. Error certification, tolerance calibration, repeated-query reuse, external solver comparisons, and 12 real experimental conditions in four physical domains support a specialized conclusion: when a useful Mittag--Leffler backbone is known, DFSC provides an efficient, error-aware route to autograd, GPU matrix-function actions, and neural residual composition. The heated-steam result also shows that structural mismatch can erase the bare propagator advantage; in that setting the hybrid offers compactness rather than best accuracy. DFSC is an internally coherent research-grade beta, not a universal solver or production ecosystem.

\section*{Acknowledgments}
This research used computational resources of Hangzhou Dianzi University.

\section*{Funding}
This work was supported by the State Key Laboratory of Ocean Engineering, Shanghai Jiao Tong University (Grant No. GKZD010089), and the State Key Laboratory of Acoustics, Chinese Academy of Sciences (Grant No. SKLA202406). The funders had no role in the study design; collection, analysis, and interpretation of data; writing of the manuscript; or decision to submit the article for publication.

\section*{Declaration of Competing Interest}
The authors declare that they have no known competing financial interests or personal relationships that could have appeared to influence the work reported in this paper.

\section*{CRediT Authorship Contribution Statement}
\textbf{Ning Hu:} Conceptualization, Methodology, Software, Validation, Formal analysis, Investigation, Data curation, Visualization, Writing--original draft, Writing--review and editing. \textbf{Haitao Duan:} Conceptualization, Methodology, Writing--review and editing. \textbf{Shuqun Li:} Validation, Investigation, Writing--review and editing. \textbf{Chuyang Hu:} Validation, Investigation, Writing--review and editing.

\appendix
\section{Proof Details and Error Decomposition}
\label{app:proofs}
Let $P_M=\Phi_M\Phi_M^{-1}$ denote projection onto the retained subspace. Applying Eq.~\eqref{eq:governing} to each retained eigenvector gives the scalar equation ${}^{C}D_t^{\alpha}\widehat u_n+\mu_n\widehat u_n=0$, whose solution is $\widehat u_n(t)=E_{\alpha}(-\mu_n t^{\alpha})\widehat u_n(0)$. Stacking the modes proves Eq.~\eqref{eq:homogeneous}. For a numerical output $\widetilde u_M$, a practical decomposition is
\begin{equation}
\|u-\widetilde u_M\|\leq
\|u-P_Mu\|+
\|P_Mu-u_M\|+
\|u_M-\widetilde u_M\|,
\end{equation}
where the terms represent spatial truncation or projection, model reduction, and numerical evaluation. The final term contains special-function, quadrature or Krylov, and floating-point contributions. Except for the restricted series remainder below, DFSC diagnoses these contributions empirically and does not report them as rigorous global bounds.

To justify Eq.~\eqref{eq:alternating-bound}, define $r_k=a_{k+1}/a_k=x\Gamma(\alpha k+\beta)/\Gamma(\alpha k+\alpha+\beta)$. For $\alpha>0$, the ratio $\Gamma(y+\alpha)/\Gamma(y)$ is increasing in $y>0$ because its logarithmic derivative is $\psi(y+\alpha)-\psi(y)>0$. Hence $r_k$ is nonincreasing. If $a_N\leq a_{N-1}$, then $r_{N-1}\leq1$ and all subsequent term magnitudes decrease. The alternating-series theorem then bounds the entire remainder by $a_N$. The implementation checks the sign, parameter domain, and entry into this decreasing regime before assigning rigorous status.

On a fixed differentiable evaluator branch, Eq.~\eqref{eq:homogeneous} is a composition of differentiable basis transforms, modal-rate maps, powers, special-function approximations, and tensor contractions. The chain rule therefore yields Eq.~\eqref{eq:chain}. For Eq.~\eqref{eq:blend}, $s'(q)=6q(1-q)$, hence $s'(0)=s'(1)=0$. The blending weight therefore contributes no endpoint slope with respect to $r$; exact derivative matching still depends on agreement of the two finite approximations and on fixed routing thresholds.

For adaptive work control, define $\Omega_j$ as an open parameter subset on which the detached stopping rule selects $w_j$. On $\Omega_j$, the returned tensor is exactly $u_{w_j}(\vartheta)$, so differentiability and the derivative identity follow from the assumed regularity of that candidate. The sets meet at equalities in the stopping criterion. Because the implementation does not differentiate the discrete index $j$, this argument supplies no matching condition for derivatives on $\partial\Omega_j$. The numerical scan in Fig.~\ref{fig:adaptive-gradient} tests representative boundaries but does not replace such a condition with a global theorem.

\section*{Declaration of Generative AI and AI-Assisted Technologies in the Manuscript Preparation Process}
During the preparation of this work, the authors used OpenAI Codex to support language editing, manuscript organization, and consistency checks across the manuscript, references, and supplementary files. The authors reviewed, edited, and verified the resulting content, including numerical claims and references, and take full responsibility for the content of the publication.

\bibliographystyle{plainnat}
\bibliography{references}

\end{document}